\documentclass[american,a4paper,runningheads,envcountsect]{llncs}

\usepackage{babel}
\usepackage[utf8]{inputenc}
\usepackage[T1]{fontenc}

\usepackage{amsmath}
\usepackage{amssymb}
\usepackage{mathtools}
\usepackage{fca}

\usepackage{amsthm}

\usepackage{csquotes}
\usepackage{booktabs}
\usepackage{enumitem}

\DeclareMathOperator{\Imp}{Imp}

\usepackage[hidelinks,bookmarks=false]{hyperref}
\usepackage{cleveref}
\let\cref\Cref
\usepackage[shrink=40]{microtype}

\newcommand{\ro}{\operatorname{\mathsf{ro}}}
\newcommand{\fl}{\operatorname{\mathsf{fl}}}
\newcommand{\ed}{\operatorname{\mathsf{ed}}}
\newcommand{\cross}{\!\bigtimes\!}

\newtheorem{thm}{Theorem}[section]
\newtheorem{prop}[thm]{Proposition}
\newtheorem{cor}[thm]{Corollary}
\newtheorem{lem}[thm]{Lemma}

\theoremstyle{definition}
\newtheorem{defn}[thm]{Definition}

\theoremstyle{remark}
\newtheorem{rem}[thm]{Remark}
\newtheorem{exa}[thm]{Example}

\usepackage{tikz}

\usepackage[backend=bibtex,style=numeric-comp,doi=false,isbn=false,%
url=false]{biblatex}
\newcommand\blfootnote[1]{%
  \begingroup
  \renewcommand\thefootnote{}\footnote{#1}%
  \addtocounter{footnote}{-1}%
  \endgroup
}

\begin{document}

\title{Towards Collaborative Conceptual Exploration}

\author{Tom Hanika\inst{1,2}
  \and Jens Zumbrägel\inst{3}}
\institute{%
  Knowledge \& Data Engineering Group\\
  University of Kassel, Germany\\
  \and
  Interdisciplinary Research Center for Information System Design\\
  University of Kassel, Germany\\[1ex]
  \and
  Faculty of Computer Science and Mathematics\\
  University of Passau, Germany\\[1ex]
  \email{tom.hanika@cs.uni-kassel.de,
    jens.zumbraegel@uni-passau.de}}

\maketitle

\blfootnote{The authors are given in alphabetical order.
  No priority in authorship is implied.}

\begin{abstract}
  In domains with high knowledge distribution a natural objective is
  to create principle foundations for collaborative interactive
  learning environments.  We present a first mathematical
  characterization of a collaborative learning group, a
  \emph{consortium}, based on closure systems of attribute sets and
  the well-known attribute exploration algorithm from formal concept
  analysis.  To this end, we introduce (weak) local experts for
  subdomains of a given knowledge domain.  These entities are able to
  refute and potentially accept a given (implicational) query for some
  closure system that is a restriction of the whole domain.  On this
  we build up a \emph{consortial expert} and show first insights about
  the ability of such an expert to answer queries. Furthermore, we
  depict techniques on how to cope with falsely accepted implications
  and on combining counterexamples.  Using notions from combinatorial
  design theory we further expand those insights as far as providing
  first results on the decidability problem if a given consortium is
  able to explore some target domain.  Applications in conceptual
  knowledge acquisition as well as in collaborative interactive
  ontology learning are at hand.
\end{abstract}


\section{Introduction}
\label{sec:introduction}

Collaborative knowledge bases, like DBpedia%
\footnote{\url{http://wiki.dbpedia.org}} and Wikidata%
\footnote{\url{http://www.wikidata.org}}~\cite{VrandecicK14},
raise the need for (interactive) collaborative tools in order to add,
enhance or extract conceptual knowledge to and from those. As well, a society
with highly specialized experts needs some method to make use of the
collaborative knowledge of those.

One particular task in knowledge acquisition is to obtain concepts in
a given domain which is composed of two disjoint sets, called
\emph{objects} and \emph{attributes}, along with some relation between
them.  A well-known approach for this is the (classical) attribute
exploration algorithm from formal concept analysis
(FCA)~\cite{fca-book,Ganter99}.  This algorithm is able to explore any
domain of the kind mentioned above by consulting some domain expert.
The result is a formal concept lattice, i.e., an order-theoretic
lattice which contains all formal concepts discovered in the
domain.  It is crucial that the algorithm has access to a \emph{domain
  expert} for the whole domain, to whom it uses a minimal number of
queries (which may still be exponential in the size of input, i.e.,
the size of the relation between objects and attributes).

However, the availability of a domain expert is often not given in
practice.  Moreover, even if it exists, such an expert might not be able
or willing to answer the possibly exponential number of queries.  The
purpose of the present work is to provide a solution in this case, at
least for some of such tasks, given a certain collaborative scenario.  More
precisely, suppose that we have a covering $M = \bigcup_{i \in I} N_i$
of the attribute set~$M$ together with a set of \emph{local experts}
$p_i$ on~$N_i$, then we propose a \emph{consortial expert} for the
domain.  As is easy to see, such an expert is in general less capable
of handling queries than a domain expert.  Nonetheless, depending on
the form of $\mathcal{M} = \{ N_i \mid i \in I \}$ our approach may
still be able to answer a significant amount of non-trivial queries.

In this work we provide a first complete characterization of (weak)
local experts in order to define what a \emph{consortium} is, what
can be explored and what next steps should be focused on.  As to our
knowledge, this has not been considered before in the realm of
conceptual knowledge.

Here is an outline of the remainder of this paper.  After giving an
account of related work in \cref{sec:related-work}, we recall basic
notions from formal concept analysis and the attribute exploration
algorithm in~\cref{sec:attr-expl}.  We define the setting of a
consortium in~\cref{sec:consortium}, using a small simplification in
notation to mere closure systems on $M$.  Subsequently we discuss our
approach in~\cref{sec:expl-with-cons}, give examples
in~\cref{ssec:abil-limits-cons}, following by possible extensions
in~\cref{sec:next-steps} and a conclusion in~\cref{sec:conclusion}.

\section{Related work}
\label{sec:related-work}

There are several related fields that address the problem of
(interactive) collaborative learning in their respective scientific
languages.  Based on modal logic there are various new approaches for
similar problems as considered here, using epistemic and intuitionistic
types.  For example, Jäger and Marti \cite{JagerM16} present a
multi-agent system for intuitionistic distributed knowledge (with
truth).  Another example is resolving the distributed knowledge of a
group as done by \AA gotnes and W\'ang~\cite{AgotnesW17}.  In this
work the process of \emph{distributed knowledge}, i.e., knowledge
distributed throughout a group, is resolved to common knowledge, i.e.,
knowledge that is known to all members of the group, a fact which is
also known to the members of the group.

Investigations considering a more virtual approach for collaborative
knowledge acquisition are, for example, presented by Stange,
N\"urnberger and Heyn~\cite{StangeNH15}, in which a collaborative
graphical editor used by experts negotiates the outcome.  Our approach
is yet based on (basic) formal attribute exploration~\cite{fca-book}.
Of course, there are various advanced versions like adding background
knowledge~\cite{Ganter99}, relational exploration~\cite{Rudolph2006}
or conceptual exploration~\cite{Stumme97}.  There are also extensions
of the basic exploration to treat incomplete
knowledge~\cite{Burmeister2005,holzerknowledge,Obiedkov2002}.

In FCA one of the first considerations on cooperatively building
knowledge bases is work of Martin and Eklund~\cite{MartinE01}.
Previous work on collaborative interactive concept lattice
modification in order to extract knowledge can be found
in~\cite{TangT13}. These concept lattice modifications are based on
removing or adding attributes/objects/concepts using expert knowledge,
and those operations may be used in a later version of collaborative
conceptual exploration.  The most recent work specifically targeting
collaborative exploration is~\cite{obiedkov16}, raising the task of
making exploration collaborative.

\section{Attribute exploration and FCA basics}
\label{sec:attr-expl}

In this paper we utilize notions from formal concept analysis (FCA) as
specified in~\cite{fca-book}.  In short, our basic data structure is a
\emph{formal context} $\mathbb{K} \coloneqq \GMI$ with~$G$ some object
set, $M$~some attribute set, and $I \subseteq G \times M$ an incidence
relation between them.  By $\cdot'$ we denote two mappings
$\cdot'\colon\mathcal{P}(G)\to\mathcal{P}(M)$ and
$\cdot'\colon\mathcal{P}(M)\to\mathcal{P}(G)$, given by
$A \mapsto A' = \{ m \in M \mid \forall g \in A \colon (g, m) \in I \}$
for $A \subseteq G$ and $B \mapsto B' = \{ g \in G \mid
\forall m \in B \colon (g, m) \in I \}$ for $B \subseteq M$.

The set $\mathfrak{B}(\mathbb{K})$ is the set of all \emph{formal
  concepts}, i.e., the set of all pairs $(A, B)$ with $A \subseteq G$,
$B \subseteq M$ such that $A' = B$ and $B' = A$.  In a formal concept
$(A, B)$ the set~$A$ is called (concept-)extent and the set~$B$ is
called (concept-)intent.  The set of all formal concepts can be
ordered by $(A, B) \le (C, D) :\Leftrightarrow A \subseteq C$.  The
ordered set $\mathfrak{B}(\mathbb{K})$, often denoted by
$\underline{\mathfrak{B}}(\mathbb{K})$, is called the \emph{concept
  lattice} of~$\mathbb{K}$.  Furthermore, the composition~$\cdot''$
constitutes closure operators on~$G$ and on~$M$, respectively, i.e.,
mappings $\cdot'' \colon \mathcal{P}(G) \to \mathcal{P}(G)$ and
$\cdot'' \colon \mathcal{P}(M) \to \mathcal{P}(M)$ which are
extensive, monotone and idempotent.  Therefore, every formal context
gives rise, through the associated closure operator, to two closure
systems, one on~$G$ and one on~$M$, called the closure system of
intents and extents, respectively.  Each of those closure systems can
be considered as an ordered set using the inclusion operator
$\subseteq$, which in turn leads to a complete lattice.  Using the
basic theorem of FCA~\cite{fca-book} one may construct for any closure
system~$\mathcal{X}$ on~$M$ a formal context~$\mathbb{K}$ such that
the closure system~$\mathcal{X}$ is the set of concept-intents
from~$\mathbb{K}$.

In the following exposition we will concentrate on the attribute
set~$M$ of a formal context.  We do this for brevity and clarity
reasons, only.  Namely, we avoid carrying all the necessary notation
through the defining parts of a collaborating consortium.  However, we
do keep in mind that~$M$ is still a part of a formal context $\GMI$,
and we rest on this classical representation, in particular, when
quoting well-known algorithms from FCA.

\subsection{Setting}

Let~$M$ be some finite (attribute) set.  We fix a closure system
$\mathcal{X} \subseteq \mathcal{P}(M)$, called the \emph{(target)
  domain} or \emph{target closure system}, which is the domain
knowledge to be acquired.  The set of all closure systems on a set $M$
constitutes a closure system itself.  In turn, this means we can also
find a concept lattice for this set.  We depict this lattice in
general in~\cref{fig:closuresys} (right).  The size of this set is enormous
and only known up to $|M|=7$.  In the next
subsection we recall the classical algorithm to compute the target
domain for a given set of attributes~$M$ using a domain expert on~$M$.
This algorithm employs rules between sets of attributes which we now
recall.  An \emph{implication} is a pair $(A, B) \in \mathcal{P}(M)
\times \mathcal{P}(M)$, which can also be denoted by $A \to B$.  We
write $\Imp(M)$ for the set $\mathcal{P}(M) \times \mathcal{P}(M)$
of all implications on~$M$.  The implication $(A, B) \in \Imp(M)$ is
\emph{valid} in~$\mathcal{X}$ if $\forall X \in \mathcal{X} :\, A
\subseteq X \,\Rightarrow\, B \subseteq X$.

\subsection{Attribute exploration}
\label{ssec:attr-expl-1}

Attribute exploration is an instance of an elegant strategy to explore the
knowledge of an (unknown) domain $\GMI$ using queries to a domain expert for
$M$. These queries consist of validity questions concerning implications in
$M$. The expert in this setting can either accept an implication, i.e.,
confirming that this implication is valid in the domain, or has to provide a
counterexample. The following description of this algorithm is gathered
from~\cite{Ganter16}, a compendium on conceptual exploration methods.

Using a signature, which specifies the logical language to be used during
exploration, there is a set of possible implications $\mathcal{F}$, each either
valid or not in the domain. The algorithm itself uses an exploration knowledge
base $(\mathcal{L},\mathcal{E})$, with $\mathcal{L}$ being the set of the
already accepted implications and $\mathcal{E}$ the set of already collected
counterexamples. These can be considered in our setting as named subsets of $M$,
where the name is the object name for this set. The algorithm now makes use of a
query engine which draws an implication $f$ from $\mathcal{F}$ that cannot be
deduced from $\mathcal{L}$ and that cannot be refuted by already provided
counterexamples in $\mathcal{E}$. This implication is presented to the domain
expert, who either can accept this implication, which adds $f$ to $\mathcal{L}$,
or refute $f$ by a counterexample $E\subseteq M$, which adds $E$ to
$\mathcal{E}$.

The crucial part here is that the domain expert has to be an expert
for the whole domain, i.e., an expert for the whole attribute set $M$
and any object possible. Otherwise, the expert would not be able to
provide complete counterexamples, i.e., the provided counterexamples
are possibly missing attributes from~$M$, or even \textquote{understand}
the query. To deal with this impractical limitation algorithms for
attribute exploration with partial (counter-)examples were introduced.
We refer the reader to~\cite[Algorithm 21]{Ganter16}.  This algorithm
is able to accept partial counterexamples from a domain expert.

The return value of the attribute exploration algorithm is the
canonical base of all valid implications from the domain. There is no
smaller set of implications than the canonical base for some closure
operator on an (attribute) set~$M$, which is sound and complete.

In the subsequent section we provide a characterization of a
consortial expert which could be utilized as such a domain expert
providing incomplete counterexamples. In addition, we show a strategy
for how to deal with counterexamples de-validating already accepted
implications, which will be a possible outcome when consulting a
consortium.

\section{Consortium}
\label{sec:consortium}


In the following we continue to utilize mere closure systems on~$M$
for some domain $\GMI$ and also call such a closure system itself the
(target) domain~$\mathcal X$, to be explored.  This ambiguity is for
brevity, only.  Furthermore, we consider~$M$ always to be finite.

\begin{defn}[Expert]
  \label{expert}
  An \emph{expert} for~$\mathcal{X}$ is a mapping $p \colon \Imp(M)
  \to \mathcal{X} \cup \{ \top \}$ such that for every $f = (A, B) \in
  \Imp(M)$ the following is true:
  \begin{enumerate}[label=\upshape(\arabic*)]
  \item \ $p(f) = \top \,\Rightarrow\, f$ is valid in $\mathcal{X}$,
  \item \ $p(f) = X \in \mathcal X \,\Rightarrow\, A \subseteq X \,\wedge\,
    B \not\subseteq X$.
  \end{enumerate}
  We refer to the set~$M$ also as the \emph{expert domain}.
\end{defn}

From this definition we note, for an implication $f \in \Imp(M)$, that
$p(f) \neq \top$ implies that~$f$ is not valid in $\mathcal{X}$, since
$p(f) = X \in \mathcal{X} \,\Rightarrow\, A \subseteq X \,\wedge\,
B \not\subseteq X$.  In analogy to this expert we now introduce an
expert on a subset of~$M$.

\begin{defn}[Local expert]
  \label{localexpert}
  Let $N \subseteq M$.  A \emph{local expert} for~$\mathcal{X}$ on~$N$
  is a mapping $p_{N} \colon \Imp(N) \to \mathcal{X}_{N} \cup \{ \top
  \}$ with $\mathcal{X}_{N} \coloneqq \{ X \cap N \mid X \in \mathcal{X} \}$
  such that for every $f = (A, B) \in \Imp(N)$ there holds:
  \begin{enumerate}[label=\upshape(\arabic*)]
  \item \ $p_{N}(f) = \top \,\Rightarrow\, f$ is valid in $\mathcal{X}$,
  \item \ $p_{N}(f) = X \in \mathcal X_N \,\Rightarrow\, A \subseteq X
    \,\wedge\, B \not\subseteq X$
  \end{enumerate}
\end{defn}

Observe that the set $\mathcal{X}_{N}$ is also a closure system.
Despite that, the elements of $\mathcal{X}_{N}$ are not necessarily
elements of $\mathcal{X}$.  But, since $N \subseteq M$ there is for
every $X \in \mathcal{X}_{N}$ some $\hat X \in \mathcal{X}$ such that
$\hat X \cap N = X$.

\begin{rem}
  \label{remark:localexpert}
  Every expert for~$\mathcal X$ provides in the obvious way a local
  expert for~${\mathcal X}$ on~$N$, for each $N \subseteq M$.
  Furthermore, every local expert for~$\mathcal{X}$ on~$N$ is a local
  expert for~$\mathcal{X}$ on~$O$ for each $O \subseteq N$.
\end{rem}

\begin{lem}[Refutation by local expert]
  Let~$\mathcal{X}$ be some domain with attribute set~$M$ and
  let~$p_{N}$ be a local expert for~$\mathcal{X}$ on $N \subseteq M$.
  Then for every $f \in \Imp(N)$ there holds $p_N(f) \neq \top
  \,\Rightarrow\, f$ is not valid in~$\mathcal{X}$.
\end{lem}

\begin{proof}
  If $p_{N}(f) \ne \top$, then $\exists X \in \mathcal{X}_{N}
  :\, p_{N}(f) = X \,\wedge\, A \subseteq X \,\wedge\, B \not\subseteq X$.
  By definition $\exists \hat X \in \mathcal{X} : \hat X \cap N = X$.
  Therefore, $A \subseteq X = \hat X \cap N \subseteq \hat X$ and
  $B \not\subseteq X = \hat X \cap N$, thus $B \not\subseteq \hat X$
  as $B \subseteq N$.
\end{proof}

\begin{exa}\label{exa:toy}
  Suppose we have a three-element attribute set $M = \{ \ro, \fl, \ed \}$,
  for the attributes “round”, “flexible” and “edible”.  Regarding the
  objects “ball”, “sphere” and “donut” (food) we consider the following
  formal context.
  \begin{center}
    \begin{tabular}{c|ccc}
      & ~round~ & ~flexible~ & ~edible~ \\\hline
      & & & \\[-4mm]
      ball & $\cross$ & $\cross$ & \\
      sphere & $\cross$ & & \\
      ~donut~ & & $\cross$ & $\cross$
    \end{tabular}
  \end{center}

  From this we obtain as our target domain
  \[ \mathcal X = \big\{ M, \{ \ro, \fl \}, \{ \fl, \ed \},
  \{ \fl \}, \{ \ro \}, \varnothing \big\} , \]
  with the canonical base $\mathcal B = \{ \ed \to \fl \}$.
  Using the shortcuts $\ed^{\complement} = \{ \ro, \fl \}$ and
  $\ro^{\complement} = \{ \fl, \ed \}$, the concept lattice may be
  depicted as:
  
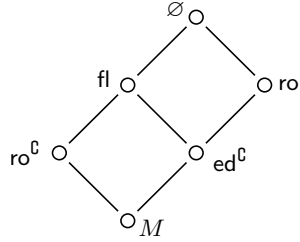
\begin{center}
    \begin{tikzpicture}[scale=0.9, semithick, font=\footnotesize]
      \node (M)   at (1, 0) {};  \node (nro) at (0, 1) {};
      \node (ned) at (2, 1) {};  \node (fl)  at (1, 2) {};
      \node (ro)  at (3, 2) {};  \node (o)   at (2, 3) {};
      \draw (M)   circle (0.1);  \draw (nro) circle (0.1);
      \draw (ned) circle (0.1);  \draw (fl)  circle (0.1);
      \draw (ro)  circle (0.1);  \draw (o)   circle (0.1);
      \draw (M)   -- (nro) -- (fl)  -- (o);
      \draw (M)   -- (ned) -- (ro)  -- (o);
      \draw (ned) -- (fl);
      \node at (1.35, -0.1) { $M$ }; \node at (2.5, 0.9) { $\ed^{\complement}$ };
      \node at (-0.5, 1) { $\ro^{\complement}$ }; \node at (3.35, 2) { $\ro$ };
      \node at (0.65, 2.1) { $\fl$ }; \node at (1.7, 3.1) { $\varnothing$ };
    \end{tikzpicture}
  \end{center}

  Now suppose that $I = \{ a, b, c \}$, and for each $i \in I$ we have a
  local expert~$p_i$ for~$\mathcal X$ on~$N_i$, where $N_a = \{ \ro, \fl \}$,
  $N_b = \{ \fl, \ed \}$ and $N_c = \{ \ro, \ed \}$.  We name the local
  experts “Alice”, “Bob” and “Carol”.

  Alice may be consulted for the implications $\ro \to \fl$ and
  $\fl \to \ro$, both of which she refutes.  For example, to the query
  $\ro \to \fl$ she responds (possibly having the sphere in mind) with
  an attribute set~$X$ containing~$\ro$ but not~$\fl$, i.e.,
  $X = \{ \ro \}$, where $\{ \ro \} = \hat X \cap \{ \ro, \fl \}$ and
  $\hat X \in \mathcal X$.  Similarly, she refutes the query
  $\fl \to \ro$ (having the donut in mind).  Moreover, local expert
  Bob can be consulted with the implications $\fl \to \ed$, which he
  refutes (ball), and $\ed \to \fl$, which he correctly accepts.
  Finally, Carol refutes both possible queries $\ed \to \ro$ (donut)
  and $\ro \to \ed$, in which case her counterexample could stem from
  different objects (ball or sphere).
\end{exa}

For some applications a local expert may be too strong in terms of having either
to accept an implication (vicariously for $\mathcal{X}$) or refute an
implication. This would require that the local expert is aware of all possible
counterexamples, which is impractical.

\begin{defn}[Local pre-expert]
  \label{def:lpe}
  A \emph{local pre-expert} for~$\mathcal{X}$ on $N \subseteq M$ is a
  mapping $p_{N}^{*} \colon \Imp(N) \to \mathcal{X}_{N} \cup \{ \top
  \}$ such that $\forall f = (A, B) \in \Imp(N) :\, p_{N}^{*}(f)
  = X \in \mathcal X_N \,\Rightarrow\, A \subseteq X \,\wedge\,
  B \not\subseteq X$.
\end{defn}

It is obvious that a local expert is also a local pre-expert.  Using
this ``weaker'' mapping we introduce the consortial (pre-)expert,
after stating what a consortial domain is and some technical result
about the intersection of closed sets.

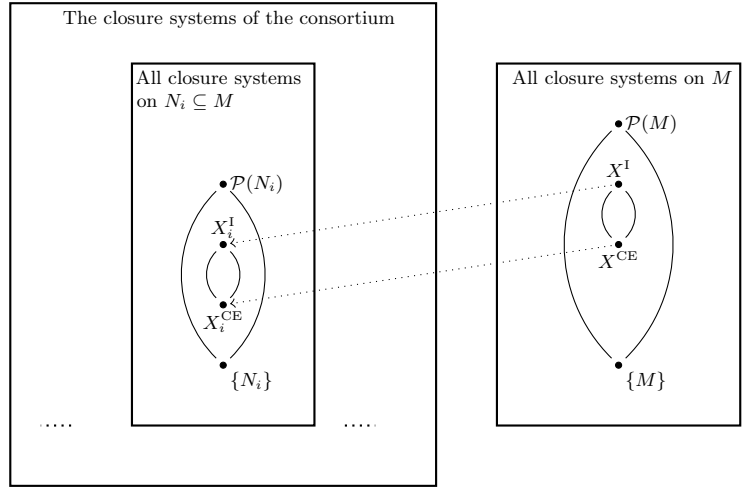
\begin{figure}[h]
  \centering
  \begin{tikzpicture}[transform shape,scale=0.8]
    \draw[thick] (-1,0) rectangle (6,8)node [below left, text width=6cm]
    {The closure systems of the consortium};
    \draw[thick] (1,1) rectangle (4,7)  node [below left, text width=2.8cm]
    {All closure systems on $N_{i}\subseteq M$};
    \draw[thick,dotted] (0,1) -- (-0.5,1);
    \draw[thick,dotted] (4.5,1) -- (5,1);
    \begin{scope}[xshift=2cm]
      \draw[thick] (5,1) rectangle (9,7) node [below left] {All closure systems on $M$};
      \draw[fill]  (7,2) circle(0.05) node (bottom) {};
      \draw[fill]  (7,6) circle(0.05) node (top) {};
      \node at (top) [right] {$\mathcal{P}(M)$};
      \node at (bottom) [below right] {$\{M\}$};
      \draw[fill]  (7,4) circle(0.05) node (bottomXI) {};
      \draw[fill]  (7,5) circle(0.05) node (topXI) {};
      \node at (topXI) [above] {$X^{\text{I}}$};
      \node at (bottomXI) [below] {$X^{\text{CE}}$};
      \draw[bend angle=45, bend left] (bottomXI) to (topXI);
      \draw[bend angle=45, bend right] (bottomXI) to (topXI);
    \end{scope}
    \draw[fill]  (2.5,3) circle(0.05) node (bottomXS) {};
    \draw[fill]  (2.5,4) circle(0.05) node (topXS) {};
    \node at (topXS) [above] {$X^{\text{I}}_{i}$};
    \node at (bottomXS) [below] {$X^{\text{CE}}_{i}$};
    \draw[bend angle=45, bend left] (bottomXS) to (topXS);
    \draw[bend angle=45, bend right] (bottomXS) to (topXS);
    \draw[bend angle=45, bend left] (bottom) to (top);
    \draw[bend angle=45, bend right] (bottom) to (top);
    \draw[fill]  (2.5,2) circle(0.05) node (bottomsmall) {};
    \node at (bottomsmall) [below right] {$\{N_{i}\}$};
    \draw[fill]  (2.5,5) circle(0.05) node (topsmall) {};
    \node at (topsmall) [right] {$\mathcal{P}(N_{i})$};
    \draw[bend angle=45, bend left] (bottomsmall) to (topsmall);
    \draw[bend angle=45, bend right] (bottomsmall) to (topsmall);
    \draw[dotted,->] (topXI) -- (topXS);
    \draw[dotted,->] (bottomXI) -- (bottomXS);
  \end{tikzpicture}
  \caption[Closure systems]{Closure system of all closure systems on~$M$
    (right) and on $N_i \subseteq M$ (left).  The closure systems for
    the set of accepted implications are denoted by~$X^{I}$ (in~$M$) and
    by~$X^{I}_{i}$ (in~$N_{i}$), and likewise for the set of
    counterexamples by~$X^{\text{CE}}$ and by~$X^{\text{CE}}_{i}$.}
  \label{fig:closuresys}
\end{figure}

\begin{defn}[Consortial domain]
  Let $M$ be some attribute set and $\mathcal{X} \subseteq
  \mathcal{P}(M)$ be the target domain.  Then a family $\mathcal{M} =
  \{ N_{i} \mid i \in I \} \subseteq \mathcal{P}(M)$ for some index
  set~$I$ is called \emph{consortial domain on $M$} if $\bigcup_{i \in
    I} N_{i} = M$.

  We call $\mathcal{M}\subseteq\mathcal{P}(M)$ a \emph{proper consortial
    domain} if $M\not\in\mathcal{M}$.
\end{defn}

\begin{lem}[Consortial domain closed under intersection]
  Let $\mathcal{M}$ be some consortial domain on $M$.
  If $\mathcal{M}$ is closed under intersection, then so is the set
  $\bigcup_{M \in \mathcal{M}} \mathcal{X}_{M}$.
\end{lem}

In the following proof as well as in the rest of this work we may
often use the abbreviation $\mathcal{X}_{i} \coloneqq \mathcal{X}_{N_{i}}$
for some $N_{i}\in\mathcal{M}$ with $\mathcal{M}$ a consortial domain
and using the notation introduced in Definition~\ref{localexpert}.

\begin{proof}
  Whenever $X \cap N_{i} \in \mathcal{X}_{i}$ and $Y \cap N_{j} \in
  \mathcal{X}_{j}$, where $X, Y \in \mathcal{X}$, we get
  \[ (X \cap N_{i}) \cap (Y \cap N_{j})
  = (X \cap Y) \cap (N_{i} \cap N_{j}) \in
  \mathcal X_{N_i \cap N_j} , \]
  where $X \cap Y \in \mathcal X$ and $N_i \cap N_j \in \mathcal M$.
\end{proof}

\begin{cor}
  If $\mathcal{M}^{*} \coloneqq \mathcal{M} \cup\{ M \}$ is a closure system,
  then so is $\bigcup_{M \in \mathcal{M}^{*}} \mathcal{X}_{M}$.
\end{cor}

By definition a proper consortial domain cannot be a closure system
and even a consortial domain will almost never have this property,
either.  However, for any consortial domain $\mathcal{M}$ we can
easily construct an intersection closed set using the downset operator
$\downarrow\! \mathcal{M}$.  Therefore, whenever we have a consortial
domain we may consider $\downarrow\! \mathcal{M}$, when necessary.
Hence, we always can construct a closure system $\mathcal{M}^{*}$ for
any consortial domain $\mathcal{M}$.

In the following we may use $M^{*}$ to speak about
$\downarrow\! \mathcal{M} \cup \{ M \}$.

\begin{rem}[Closure operator $\mathcal{M}^{*}$]
  Since for a given consortial domain $\mathcal{M}$ on $M$ the set
  $\mathcal{M}^{*}$ is a closure system, we obtain a closure operator
  $\phi:\mathcal{P}(M)\to\mathcal{P}(M)$. We may address $\phi$ simply by
  $\mathcal{M}^{*}()$ and the image of $N\subseteq M$ by $\mathcal{M}^{*}(N)$.
\end{rem}

Using the just discovered closure operator we may define which queries can be
answered in a consortial domain.

\begin{defn}[Well-formed query]
  Let $\mathcal M$ be some proper consortial domain on~$M$ and let $f
  = (A, \{ b \}) \in \Imp(M)$.  Then~$f$ is called \emph{well-formed}
  for $\mathcal{M}$ if $\mathcal{M}^{*}(A \cup \{ b \}) \neq M$, i.e.,
  if there exists $N_i \in \mathcal M$ such that $A \cup \{ b \}
  \subseteq N_i$.
\end{defn}

Well-formed queries are in fact the only queries for which in a proper
consortial domain the decision problem if an implication is valid can
be decided.  It is easy to see that for any given $f=(A,\{b\}) \in
\Imp(M)$, if $\mathcal{M}^{*}(A\cup\{b\}) = M$, then there is no expert
domain left, therefore either the conclusion attribute or one of the
premises is missing in all $N\in\mathcal{M}$, which leads to
undecidability.

Putting all those ideas together we are finally able to define our
main goal.

\begin{defn}[Consortium for~$\mathcal X$]
  For an attribute set~$M$ and a target domain $\mathcal{X}$ on~$M$
  let $\mathcal{M} = \{ N_{i} \mid i \in I \}$ be a consortial domain
  on~$M$.  A \emph{consortium} for $\mathcal{X}$ is a
  family $\mathcal{C} \coloneqq \{ p_{i} \}_{i \in I}$ of local
  pre-experts~$p_{i}$ for $\mathcal{X}$ on $N_{i}$.
\end{defn}

All comments made before about $\mathcal{M}$ being intersection-closed
are compatible with the definition of a consortium.  Using
Remark~\ref{remark:localexpert} we can always obtain a local
pre-expert for any $\hat M \in\, \downarrow\! \mathcal{M}$.
A consortium is able to decide the validity of any well-formed query, by
definition. Therefore, a consortium gives rise to a \emph{consortial expert}.
As long as all queries are well-formed, a consortium can be used in-place of a
domain expert.

\begin{exa}
  We continue with Example~\ref{exa:toy}.  On the consortial domain
  $\mathcal M \coloneqq \{ N_a, N_b, N_c \}$ the three local experts
  form a consortium $\mathcal C \coloneqq \{ p_a, p_b, p_c \}$
  for~$\mathcal X$.  Note that the consortium cannot decide, e.g., on
  the implication $\{ \fl, \ed \} \to \ro$, since this query is not
  well-formed for~$\mathcal M$.  However, if experts are able to
  combine their counterexamples they may refute the query (cf.\ \cref{sec:next-steps}).
\end{exa}

\begin{defn}[Strong consortial expert]
  \label{sconexp} 
  Let $\mathcal{C}=\{p_{i}\}_{i\in I}$ be a consortium for~$\mathcal{X}$
  on~$M$.  A \emph{strong consortial expert} is a mapping
  $p_{\mathcal{C}} \colon \bigcup_{i \in I} \Imp(N_i) \to
  \bigcup_{i \in I} \mathcal{X}_{i}\cup\{\top\}$
  such that for every $f = (A,B) \in \bigcup_{i \in I} \Imp(N_i)$
  there holds:
  \begin{enumerate}[label=\upshape(\arabic*)]
  \item $\exists p_{i} \in \mathcal{C}, \, p_{i}(f) \ne \top
    \,\Rightarrow\, p_{\mathcal{C}}(f) \ne \top$,
  \item $p_{\mathcal{C}}(f) = X \in \bigcup_{i \in I} \mathcal X_i
    \,\Rightarrow\, A \subseteq X \,\wedge\, B \not\subseteq X$.
  \end{enumerate}
\end{defn}

The strong consortial expert has to respect a possible counterexample entailed
in the consortium in order to be consistent with \cref{def:lpe}, since every
counterexample by a local (pre-)expert is a restriction of an element of the
target closure system. In the case of having local experts in the consortium
this behavior may be in conflict with~\cref{localexpert}, since we demand that
accepting an implication by a local expert implies that the implication is true
in the target domain. For example, if a local expert accepts an implication and
another local (pre-)expert refutes it, this conflict is not resolvable.
Therefore, whenever a consortium does contain local experts it is mandatory
that they meet a consistency property. We will introduce consistency
in~\cref{ssec:consistent-experts}. When using a consortium of proper local
pre-experts there is no implication from accepting an implication. An accepted
implication may be false in the target domain.

To meet our goal of reducing the number of inquiries to the individual
expert in a consortium, the proposed consortial expert from~\cref{sconexp}
is insufficient.  We need to diminish the strong requirement from
checking all experts for having a counterexample.  This leads to the
following definition.

\begin{defn}[Consortial expert]
  \label{conexp} 
  Let $\mathcal{C} = \{ p_{i} \}_{i \in I}$ be a consortium for~$\mathcal{X}$
  on~$M$.  A \emph{consortial expert} is a mapping $p_{\mathcal{C}}
  \colon \bigcup_{i \in I} \Imp(N_i) \to \bigcup_{i \in I} \mathcal{X}_{i}
  \cup \{ \top \}$ such that for every $f = (A, B) \in \bigcup_{i\in I}
  \Imp(N_i)$ there holds:
  \begin{enumerate}[label=\upshape(\arabic*)]
  \item $\exists p_{i} \in \mathcal{S}, \, p_{i}(f) \ne \top
    \,\Rightarrow\, p_{\mathcal{C}}(f) \ne \top$,
  \item $p_{\mathcal{C}}(f) = X \in \bigcup_{i \in I} \mathcal X_i
    \,\Rightarrow\, A \subseteq X \,\wedge\, B \not\subseteq X$.
  \end{enumerate}
  The set $\mathcal{S} \subseteq \mathcal{C}$ is a per inquiry chosen
  subset of local experts such that $f \in \Imp(N_i)$ for every
  $p_{i} \in \mathcal{S}$.  
\end{defn}

We left the just addressed expert subset vague by intention. In
practice, choosing this should be possible in various ways.  There is
no further restriction then of choosing \textquote{qualified} experts,
i.e., how the consortial expert is choosing~$\mathcal{S}$.  One
obvious choice would be to consult all local (pre-)experts at once.  A
more clever strategy would be to consult experts covering the
attributes in question having the largest attribute size to cover in
general.  One may also employ a cost function, which could lead to
asking only less expensive experts.  While using a consortial expert
for exploration, an already accepted implication may be refuted later
on in the exploration process.  Whenever an inquiry leads to an
counterexample which is also an counterexample for an already accepted
implication, the set of valid implications needs to be corrected.


So far we provided neither constraints nor constructions about the
decision making of a consortium, i.e., the \emph{collaboration}.  The
most simple case, where~$\mathcal{M}$ is a partition of $M$ and all
queries are concerned with an element of~$\mathcal{M}$, can easily be
treated: For every query the expert for the according element
of~$\mathcal{M}$ either refutes or maybe accepts.  Since this case
seems artificial we will investigate different approaches of
\textquote{real} collaboration in the following section.

\section{Exploration with consortial experts}
\label{sec:expl-with-cons}

In general, for exploring a domain using attribute exploration with
partial examples one may use instead of the domain expert some
(strong) consortial expert.  However, there are three possible
problems to deal with.  First, a query may concern some implication~$f$
that is not well-formed for the consortium $\mathcal{C}$ that is used
by the consortial expert.  Second, if a consortium containing local
pre-experts does accept an implication this does not necessarily imply
the implication in question to be valid in the domain.  Obviously,
this also depends heavily on how a consortial expert utilizes a
consortium.  We deal with related problems in the following subsections.
Third, while choosing a subset of~$\mathcal{C}$ the consortial expert
may have missed a local pre-expert which would have been aware of a
counterexample, in contrast to a strong consortial expert.

The first problem cannot be resolved by the consortial expert. When no
local (pre-)expert can be consulted for some implication the only
choice is to accept $f$. However, a more suitable response would be a
third type of replying like \emph{NULL}. Then, the exploration
algorithm could cope with this problem by deferring to other
questions. The attribute exploration algorithm with partial examples
from~\cite{Ganter16} but could easily be adapted for this. In turn,
the algorithm would only be able to return an interval of closure
systems, like in~\cref{fig:closuresys}.

For the second problem one needs to repair the set of accepted
implications in case a counterexample turns up later in the process.
We show a method of doing so in~\cref{sec:next-steps}.  Of
course, there is still the possibility that an accepted not valid
implication will never be discovered as a consequence of an incapable
consortium.  This leads the exploration algorithm to return not the
target domain but another closure system.  How \textquote{close} this
closure system is to the target domain, in terms of some Jaccard-like
measure, is to be investigated in some future work.

The third and last problem can always be dealt with by employing a strong
consortial expert. A less exhaustive method could be to incorporate statistical
methods for quantifying the number of necessary experts to consult in order
to obtain a low margin of error.

\subsection{Correcting falsely accepted implications}
\label{ssec:corr-fals-accept}

A major issue while using a consortial expert for exploration is the
possibility of wrongly accepting an implication.  This can be dealt
with on side of the exploration algorithm.
While receiving a new counterexample $O \subseteq M$ from the
consortial expert the exploration algorithm has also to check if~$O$
is a counterexample to an already accepted implication in
$\mathcal{L}$.  When such an implication $f=(A,B)$ is found, we would
need to restrict the conclusion of $f$ to a yet not disproved subset
and also add implications with stronger premises that were omitted
because $f$ was (wrongly) accepted.  In particular, we would
replace~$f$ in $\mathcal{L}$ by $A \to B \cap C$ and also add
implications $A \cup \{ m \} \to B$ for $m \in M \setminus(A \cup C)$
to $\mathcal{L}$.

This approach may drastically increase the size of the set of already
accepted implications.  Unlike the classical exploration algorithm,
this modified version would return a very large set of implications
instead of the canonical base. One may cope with that by
utilizing~\cite[Algorithm 19]{Ganter16} after every event of replacing
an implication in $\mathcal{L}$.  This algorithm takes a set of
implications and returns the canonical base.  After this a next query
can be computed based on the so far collected set of implications and
the already collected counterexamples.

\subsection{Consistency}
\label{ssec:consistent-experts}

So far we characterized what local pre-experts and consortia are, by their
ability to make decisions about queries. In this section we provide ideas for
a consistent consortium. We start with resolving a possible conflict for
consortial experts.

\begin{defn}[Consistent experts]
  \label{consisten-experts}
  Let $\mathcal{C} = \{ p_{i} \}_{i \in I}$ be a consortium for~$\mathcal{X}$
  on~$M$ and let $\check{\mathcal{C}} \subseteq \mathcal{C}$ be the set of
  local experts in~$\mathcal{C}$.  We say that~$\mathcal{C}$ has
  \emph{consistent experts} if for $i, j \in I$ with $p_{i}, p_{j} \in
  \check{\mathcal{C}}$ and for all $f \in \Imp(N_{i} \cap N_{j})$ it holds
  that $p_{i}(f) = \top \,\Leftrightarrow\, p_{j}(f) = \top$.

  We call $\mathcal{C}$ with consistent experts a \emph{consistent experts
    consortium}.
\end{defn}

This idea from consistent experts does still allow for different local experts
to be able to refute an implication with different counterexamples. But whenever
one local expert would accept an implication, any other local expert needs to do
so as well. Different local (pre-)experts may have access to disjoint sets of
counterexamples, by design. Furthermore, local pre-experts may not have the
knowledge for all possible counterexamples in their restriction of the target
domain. Therefore, accepting an implication by a local pre-expert has no
implication itself. Hence, even in a consistent experts consortium it is still
possible that some local experts may provide a counterexample while others do
not. A stronger notion of consistency would be to forbid that.

\begin{defn}[Consistent consortium]
  Let $\mathcal{C} = \{ p_{i} \}_{i \in I}$ be a consortium for~$\mathcal{X}$
  on~$M$.  The consortium~$\mathcal{C}$ is \emph{consistent} if for all
  $i, j\in I$ and for all $f \in \Imp(N_{i} \cap N_{j})$ we have that
  $p_{i}(f) = \top \,\Leftrightarrow\, p_{j}(f) = \top$.
\end{defn}

Again, in consequence, all local pre-experts are either able to produce some,
but not necessarily the same, counterexample for some implication or all do
accept. We look into the possibility of combining counterexamples
in~\cref{sec:next-steps}.

\subsection{Abilities and limitations of a consortium}
\label{ssec:abil-limits-cons}

In this section we exhibit the theoretical abilities and limitations
of a consortium for determining the whole target domain of available
knowledge.  After clarifying some general notions and facts, we state
a reconstructability result for consortia based on combinatorial
designs.

Let us, as before, fix a finite (attribute) set~$M$.  As is
well-known, any set $\mathcal F \subseteq \Imp(M)$ of implications
constitutes a closure system
\[ \mathcal X_{\mathcal F} \coloneqq \big\{ X \in \mathcal P(M) \mid
\forall f = (A, B) \in \mathcal F :\, A \subseteq X \,\Rightarrow\, B
\subseteq X \big\} . \]
Conversely, any closure system~$\mathcal X$ defines its set
$\mathcal F_{\mathcal X} \subseteq \Imp(M)$ of valid implications, and we
have $\mathcal X_{\mathcal F_{\mathcal X}} = \mathcal X$ and
$\mathcal F_{\mathcal X_{\mathcal F}} = \mathcal F$.
Now suppose that~$\mathcal S$ is a class of closure systems
$\mathcal X \subseteq \mathcal P(M)$ on~$M$ which contains the target
domain.  This set~$\mathcal S$ describes some information on the
target domain we may have in advance.  Suppose that
$\mathcal M = \{ N_i \mid i \in I \}$ is a consortial domain and we
have, for some $\mathcal X \in \mathcal S$, a set of local experts
$p_i \colon \Imp(N_i) \to \mathcal X_{N_i} \cup \{ \top \}$ on
$N_i \in \mathcal M$, so that in particular, $p_i(f) = \top$ if and
only if~$f$ is valid in~$\mathcal X$.  Then we consider the set
\[ \mathcal F_{\mathcal M} \,\coloneqq\, \big\{ f \in
\textstyle\bigcup_{i \in I} \Imp(N_i)
\mid f \text{ is valid} \big\} \,\subseteq\, \mathcal F_{\mathcal X} \,, \]
i.e., the set of all well-formed valid implications, and let
$\mathcal X_{\mathcal M} \coloneqq \mathcal X_{\mathcal F_{\mathcal M}}$,
which is the closure system reconstructible by the consortium.
Clearly, $\mathcal X_{\mathcal M} \supseteq \mathcal X$, and from
the preceding discussion we easily deduce the following result.

\begin{prop}[Ability of a consortium]
  The consortial domain~$\mathcal M$, together with local experts
  $p_i \colon \Imp(N_i) \to \mathcal X_{N_i} \cup \{ \top \}$
  for $N_i \in \mathcal M$, is able to reconstruct the target
  domain~$\mathcal X$ within a class~$\mathcal S$ of closure
  systems on~$M$ if and only if $\mathcal Y_{\mathcal M} =
  \mathcal X_{\mathcal M}$ implies $\mathcal Y = \mathcal X$,
  for every $\mathcal Y \in \mathcal S$.
\end{prop}

\begin{exa}
  We illustrate these notions with two simple extreme cases.
  \begin{enumerate}
  \item Suppose that $\mathcal X = \{ M \}$, then every implication
    is valid, i.e., $\mathcal F_{\mathcal X} = \Imp(M)$.  Since every
    consortial domain $\mathcal M = \{ N_i \mid i \in I \}$ has the
    covering property $\bigcup_{i \in I} N_i = M$, it follows that
    $\mathcal X_{\mathcal M} = \mathcal X$.  Hence, if
    $\mathcal Y_{\mathcal M} = \mathcal X_{\mathcal M}$, then $
    \mathcal Y_{\mathcal M} = \{ M \}$, so that $\mathcal Y = \{ M \}
    = \mathcal X$, i.e., the consortium is always able to
    reconstruct~$\mathcal X$ in the class of all closure systems on~$M$.
  \item Consider the case $\mathcal X = \mathcal P(M)$ and suppose that
    $\mathcal M = \{ N_i \mid i \in I \}$ is a proper consortial domain.
    Then for any $m \in M$ we have $M \!\setminus\! \{ m \} \to \{ m \}
    \notin \bigcup_{i \in I} \Imp(N_i)$, whence $\mathcal X_{\mathcal M}
    = \mathcal Y_{\mathcal M}$ for $\mathcal Y = \mathcal P(M) \setminus
    \{ M \!\setminus\! \{ m \} \} \ne \mathcal X$.  Thus no proper
    consortium is capable of reconstructing the target domain.
  \end{enumerate}
\end{exa}

Let us define for a set of implications $\mathcal F \subseteq \Imp(M)$
the \emph{premise complexity} to be
$c(\mathcal F) \coloneqq \max \{ | A | \mid f = (A, B) \in \mathcal F \}$
if $\mathcal F \ne \varnothing$ and $c( \varnothing ) \coloneqq -1$.  Also,
we associate to a closure system $\mathcal X \subseteq \mathcal P(M)$
on~$M$ its \emph{premise complexity} by $c(\mathcal X) \coloneqq \min
\{ c(\mathcal F) \mid \mathcal X_{\mathcal F} = \mathcal X \}$,
which equals the premise complexity of its canonical base.

\begin{exa}
  For the extreme closure systems we have $c( \mathcal P(M) ) = -1$
  and $c(\{ M \}) = 0$.  Considering the closure system $\mathcal X_k
  \coloneqq \{ X \in \mathcal P (M) \mid |X| \le k \} \cup \{ M \}$
  we see that $c(\mathcal X_k) = k + 1$.
\end{exa}

Denote by~$\mathcal S_k$ the class of all closure systems up to
premise complexity~$k$.

\begin{thm}[Reconstructability in bounded premise complexity]
  \label{thm:steiner}
  A consortium of local experts on the consortial domain~$\mathcal M$
  is able to reconstruct a target domain within the class~$\mathcal S_k$
  if and only if every subset $O \subseteq M$ of size $k \!+\! 1$ is
  contained in some $N \in \mathcal M$.
\end{thm}

\begin{proof}
  First suppose that each subset $O \subseteq M$ of size $k \!+\! 1$
  is contained in some $N \in \mathcal M$.  We claim that $\mathcal
  X_{\mathcal M} = \mathcal X$ for every closure system $\mathcal X
  \in \mathcal S_k$, whence every target domain is reconstructible
  within~$\mathcal S_k$.  Let $\mathcal X \in \mathcal S_k$, then
  there is a set~$\mathcal F$ of implications with premise complexity
  $c(\mathcal F) \le k$ such that $\mathcal X = \mathcal X_{\mathcal
    F}$.  We may assume that each implication $f \in \mathcal F$ is of
  the form $f = (A, \{ b \})$.  By assumption there holds
  $\mathcal F \subseteq \bigcup_{N \in \mathcal M} \Imp(N)$, so that
  $\mathcal F \subseteq \mathcal F_{\mathcal X} \,\cap\, \bigcup_{N \in
    \mathcal M} \Imp(N) = \mathcal F_{\mathcal M} \subseteq \mathcal
  F_{\mathcal X}$.  This implies $\mathcal X_{\mathcal F} = \mathcal
  X_{\mathcal F_{\mathcal M}}$, i.e., $\mathcal X_{\mathcal M} =
  \mathcal X$, as desired.

  Conversely, suppose there exists a subset $O \subseteq M$ of size $k
  \!+\! 1$ not contained in any $N \in \mathcal M$.  Choose some $b
  \in O$, let $A \coloneqq O \setminus \{ b \}$, so that $|A| = k$, and
  consider the implication $f \coloneqq (A, \{ b \})$.  Then we have $f \notin
  \bigcup_{N \in \mathcal M} \Imp(N)$.  Now letting $\mathcal X \coloneqq
  \mathcal P(M)$ and $\mathcal Y \coloneqq \mathcal X_{\{ f \}}$ we then
  have distinct $\mathcal X, \mathcal Y \in \mathcal S_k$ with
  $\mathcal X_{\mathcal M} = \mathcal P(M) = \mathcal Y_M$, showing
  that~$\mathcal X$ cannot be reconstructed within~$\mathcal S_k$.
\end{proof}

Suppose that $|M| = m$.  Recall (cf.\ \cite[Sec.~2.5]{MacWilliams77})
that a \emph{Steiner system} $S(t, n, m)$ is a collection
$\{ N_i \mid i \in I \} \subseteq \mathcal P(M)$ of $n$-element
subsets $N_i \subseteq M$ such that every $t$-element subset of~$M$ is
contained in exactly one subset~$N_i$.  In light of
Theorem~\ref{thm:steiner} it is clear that the Steiner systems
$S(k \!+\! 1, m, n)$ are the minimal consortial domains that are able to
reconstruct target domains within the class~$\mathcal S_k$ of
bounded premise complexity~$k$.

\subsection{Extensions}
\label{sec:next-steps}

\subsubsection{Combining counterexamples}

A lack of our consortium setting so far is the inability to recognize similar
counterexamples. Combining counterexamples is a powerful idea that lifts the
consortium above the knowledge of the \textquote{sum} of knowledge of the local
\mbox{(pre-)experts}.

For this we need to augment a consortium by a background ontology of
counterexamples. The most simple approach would be to identify two
counterexamples from two different local (pre-)experts by matching the names of
the counterexamples, which the experts would need to provide as well. In basic
terms of FCA, while providing counterexamples the consortial expert needs to
know if the counterexamples provided by the local (pre-)experts,
restricted to their attribute sets, are of the same counterexample in
the domain.
We motivate this extension by an example.  Given we want to explore some
domain about animals with the attribute set being $M = \{
\text{mammal} ,\, \text{does not lay eggs} ,\, \text{is not poisonous} \}$
using a set of two local (pre-)experts with $N_{1} = \{ \text{mammal}
,\, \text{does not lay eggs} \}$ and $N_{2} = \{ \text{mammal} ,\,
\text{is not poisonous} \}$.  Only expert~$p_{1}$ can be consulted for
the validity of $\{ \text{mammal} \} \to \{ \text{does not lay eggs}
\}$.  Of course, $p_{1}$ refutes this implication by providing
the set $\{ \text{mammal} \}$, which he could name for example
\emph{platypus}\footnote{a semiaquatic egg-laying mammal endemic to
  eastern Australia}.  While exploring, the next query could be
$\{ \text{mammal} \} \to \{ \text{is not poisonous} \}$.  Note that
this is not answered by the counterexample of~$p_{1}$ since
$\{ \text{is not poisonous} \}$ is no subset of~$N_{1}$.  The local
(pre-)expert~$p_{2}$ refutes this of course as well, by providing
the counterexample $\{ \text{mammal} \}$ and naming this counterexample
also platypus.  Instead of collecting two different counterexamples
we are now able to combine those two and say $\{ \text{mammal} \}$ is
not just an element of~$\mathcal{X}_{1}$ and~$\mathcal{X}_{2}$ but as
well an element of~$\mathcal{X}$.  In turn, the set of counterexamples
the exploration algorithm is using contains now a more powerful
counterexample than any local expert in the consortium could have
provided. There are various ways to implement this combining of counterexamples.
For example, after acquiring a counterexample for an implication from
some expert one may ask all experts if they are aware of this
counterexample name and if they could contribute further attributes
from their local attribute set. To investigate efficient strategies to
do that is referred to future work.

\subsubsection{Coping with wrong counterexamples}

Another desirable ability for a real world consortium would being able
to handle wrong counterexamples, or more generally, having a measure
that reflects the trust a consortial expert has in counterexamples of
particular local (pre-)experts.  Our setting for a consortium is not
capable of doing this.  In fact, the consortium cannot refute an
implication using a wrong counterexample by design, since
every~$\mathcal{X}_{i}$ is a restriction of the target domain.  Hence,
all counterexamples provided by a local (pre-)expert are
\textquote{true}.  In order to allow for a consortium to provide
wrong counterexamples, one has to detach the closure system of some
expert~$p_{i}$ from the target domain~$\mathcal{X}$.  This would also
extend the possibilities of treating counterexamples by the consortial
expert. Resolution strategies from simple majority voting up to
minimum expert trust or confidence could be used.

\section{Conclusion and outlook}
\label{sec:conclusion}

In this paper we gave a first characterization of how to distribute
the r\^ole of an domain expert for attribute exploration onto a
consortium of local (pre-)experts.  Besides practically using this
method this result may be applied to various other tasks in the realm
of FCA. It is obvious that the shown approach can easily be adapted
for object exploration, the dual of attribute exploration.  Hence,
having object and attribute exploration through a consortium, we
provided the necessary tools such that collaborative concept
exploration (CCE) is at reach.  Since CCE relies on
both kinds of exploration in an alternating manner, the logical next
step is to investigate what can be explored using a consortium.  In
addition we showed preliminary results on how to evaluate a
consortium, how to shape it, i.e., how to choose a consortium from a
potentially bigger set of experts, how to treat mistakenly accepted
implications and how to increase the consistency.

Further research on this could focus on formalizing the depicted
extensions from~\cref{sec:next-steps}, where the task of modifying the
consortium in order to encounter and compute wrong counterexamples
seems as inevitable as it is hard to do.  An easier extension that
increases the ability of exploring a domain seems to be a
\textquote{clever} combining of counterexamples.

\subsubsection{Acknowledgments}

The authors would like to thank Daniel Borchmann and Maximilian Marx
for various inspiring discussions on the topic of consortia while
starting this project.  In particular, the former suggested the name
\emph{consortium} and always is the best critic one can imagine.
Furthermore, we are grateful for various challenging discussions with
Sergei Obiedkov, including ideas for coping with wrongly accepted
implications.

\appendix

\sloppy

\printbibliography

\end{document}